\documentclass[sigconf]{acmart}

\usepackage{microtype}
\usepackage[show]{chato-notes}
\usepackage{xcolor}
\usepackage{enumerate}
\usepackage[linesnumbered,noend,boxruled,noline,nofillcomment]{algorithm2e}
\usepackage{wasysym}
\usepackage{mathtools}
\usepackage{adjustbox}

\newcommand{\spara}[1]{\smallskip\noindent{\bf #1}}

\DeclarePairedDelimiter\floor{\lfloor}{\rfloor}
\newtheorem{theorem}{Theorem}[section]
\newtheorem{example}{Example}
\newtheorem{definition}{Definition}

\newtheorem{corollary}{Corollary}[section]
\newtheorem{lemma}{Lemma}[section]
\newtheorem{problem}{Problem}

\DeclareRobustCommand{\calA}[0]{{\mathcal A}}

\DeclareRobustCommand{\calR}[0]{{\mathcal R}}
\DeclareRobustCommand{\calS}[0]{{\mathcal S}}

\DeclareRobustCommand{\calU}[0]{{\mathcal U}}

\DeclareMathOperator*{\argmax}{arg\,max}
\DeclareMathOperator*{\support}{support}

\DeclareMathOperator*{\softmin}{softmin}

\newcommand{\cov}[2]{{#1}[{#2}]}
\DeclareRobustCommand{\instance}[0]{{\mathcal T}}
\DeclareRobustCommand{\users}[0]{\calU}
\DeclareRobustCommand{\sols}[0]{\calS}

\newcommand{\scov}[1]{{#1\uparrow\,}}

\newcommand{\NPhard}{$\mathbf{NP}$-hard}

\newcommand{\squishlist}{
 \begin{list}{$\bullet$}
  {  \setlength{\itemsep}{0pt}
     \setlength{\parsep}{3pt}
     \setlength{\topsep}{3pt}
     \setlength{\partopsep}{0pt}
     \setlength{\leftmargin}{2em}
     \setlength{\labelwidth}{1.5em}
     \setlength{\labelsep}{0.5em}
} }
\newcommand{\squishlisttight}{
 \begin{list}{$\bullet$}
  { \setlength{\itemsep}{0pt}
    \setlength{\parsep}{0pt}
    \setlength{\topsep}{0pt}
    \setlength{\partopsep}{0pt}
    \setlength{\leftmargin}{2em}
    \setlength{\labelwidth}{1.5em}
    \setlength{\labelsep}{0.5em}
} }

\newcommand{\squishdesc}{
 \begin{list}{}
  {  \setlength{\itemsep}{0pt}
     \setlength{\parsep}{3pt}
     \setlength{\topsep}{3pt}
     \setlength{\partopsep}{0pt}
     \setlength{\leftmargin}{1em}
     \setlength{\labelwidth}{1.5em}
     \setlength{\labelsep}{0.5em}
} }

\newcommand{\squishend}{
  \end{list}
}

\newcommand{\mycomment}[1]{}

\newcommand{\reals}{{\mathbb R}}

\newcommand{\naturals}{{\mathbb N}}

\DeclareMathOperator*{\expect}{\mathbb{E}}

\newcommand{\rst}[1]{{\ensuremath{{\mathbin|}\raise-.9ex\hbox{${\scriptstyle{#1}}$}}}}
\newcommand{\subst}[1]{{\ensuremath{\raise-.9ex\hbox{${\scriptstyle{#1}}$}}}}

\newenvironment{prooftext}[1]{\par\noindent{\bf Proof#1.}\quad}{\nopagebreak$\qed$\\}
\newenvironment{proofof}[1]{\begin{prooftext}{ of #1}}{\end{prooftext}}

\AtBeginDocument{%
  \providecommand\BibTeX{{%
    \normalfont B\kern-0.5em{\scshape i\kern-0.25em b}\kern-0.8em\TeX}}}

\copyrightyear{2021}
\acmYear{2021}
\setcopyright{acmlicensed}\acmConference[KDD '21]{Proceedings of the 27th ACM SIGKDD Conference on Knowledge Discovery and Data Mining}{August 14--18, 2021}{Virtual Event, Singapore}
\acmBooktitle{Proceedings of the 27th ACM SIGKDD Conference on Knowledge Discovery and Data Mining (KDD '21), August 14--18, 2021, Virtual Event, Singapore}
\acmPrice{15.00}
\acmDOI{10.1145/3447548.3467349}
\acmISBN{978-1-4503-8332-5/21/08}
\begin{document}
\fancyhead{}
\title{Maxmin-Fair Ranking: Individual Fairness under Group-Fairness Constraints}

\author{David Garc\'ia-Soriano }
\affiliation{%
\institution{ISI Foundation, Turin, Italy}
}
\email{d.garcia.soriano@isi.it}

\author{Francesco Bonchi}
\affiliation{%
	\institution{ISI Foundation, Turin, Italy}
\institution{Eurecat, Barcelona, Spain}
}
\email{francesco.bonchi@isi.it}

\begin{abstract}
We study a novel problem of fairness in ranking aimed at minimizing
the amount of individual unfairness introduced when enforcing group-fairness constraints.
Our proposal is rooted in the \emph{distributional maxmin fairness} theory, which uses randomization to maximize the expected satisfaction of the worst-off individuals.
We devise an exact polynomial-time algorithm 
to find maxmin-fair distributions of
general search problems (including, but not limited to, ranking),
        and show that our algorithm can produce rankings which, while satisfying the given
group-fairness constraints, ensure that the maximum possible value is brought to individuals.
\end{abstract}


\maketitle \sloppy

\section{Introduction}
\label{sec:intro}

As the position in a ranking influences to a great extent the amount of attention that an item receives,
biases in ranking can lead to unfair distribution of exposure, thus producing substantial economic impact.
If this is important when ranking items (e.g., web pages, movies, hotels, books), it raises even more crucial
concerns when ranking people. In fact, ranking is at the core of many decision-making processes in spheres such as health (e.g., triage in pandemic), education (e.g.,  university admission),  or employment (e.g., selection for a job), which can have a direct tangible impact on people's life. These concerns have captured the attention of researchers, which have thus started devising ranking systems which are \emph{fair}  for the items being ranked \cite{YangS17,CelisSV18,ZehlikeB0HMB17,SinghJ18,YangGS19,AsudehJS019,GeyikAK19}.

The bulk of the algorithmic fairness literature deals with \emph{group fairness} along the lines of \emph{demographic parity} \cite{fairness_awareness} or \emph{equal opportunity} \cite{HardtPNS16}:
this is typically expressed by means of some fairness constraint requiring that the top-$k$ positions (for any $k$) in the ranking contain enough elements from some  groups that are protected from discrimination based on sex, race, age, etc.
In fact, ~\cite{CelisMV20} shows that in a certain model, \emph{group-fairness constraints can eliminate the bias implicit in the ranking
    scores}.
    Besides, some legal norms enforce these constraints
\cite{norms1,norms2}.
For these reasons we will consider a ranking \emph{valid} if it satisfies
a given set of group-fairness constraints of this type, as detailed in Section~\ref{sec:problem}.

More formally, consider a set of elements (items or individuals) to be ranked $\users = \{u_1, \ldots, u_n\}$, a partition of $\users$ into groups defined by some protected attributes,
    and a relevance score $R: \users \rightarrow \reals^{\ge 0}$ for each element. For instance, $\users$ could be the result of a query while $R$ represents the relevance of each item
    for the query, or $\users$ could be the set of applicants for a job while $R$ their fitness for the job.
    Let $\mathcal{R}$ denote all possible rankings of $\users$ (bijections from $\calU$ to $[n]$), where $r(u) \in [n]$ denotes the
    position of element $u$ in a ranking $r \in \calR$ and let $\calS \subseteq \calR$ denote the subset of valid
        rankings satisfying the  agreed-upon constraints.
    Let $W(r,u)$ denote the
    \emph{utility} that placing $u$ at position $r(u)$ brings to the overall ranking: this is
   typically a function of the relevance score $R$, so that having higher relevance elements at top positions is rewarded. In other words, $W$ is such that, if  $r^*$ denotes the ranking by decreasing $R$, then $r^*$ is also the ranking maximizing the total utility (the so-called \emph{Probability Ranking Principle} \cite{prp} in Information Retrieval).
As the maximum-utility ranking $r^*$ might not satisfy the given group-fairness constraint, the problem typically addressed in the
literature is to find a valid ranking
which
    maximizes the global utility, i.e.,

\begin{equation}\label{eq:global}
{\tilde{r}}  \in \argmax_{r \in \mathcal{S}} \;\sum_{u \in \users} W(r,u).
\end{equation}

\begin{table}[h!]
  \centering
  \caption{Example instance. Top row: identifiers and protected attribute (gender). Bottom row: relevance score $R$.}\label{tab:running_example}
  \vspace{-3mm}
\begin{tabular}{cccccccc}
  \hline
\toprule
  $u_1, \mars$ & $u_2, \mars$ & $u_3, \venus$ & $u_4, \mars$ & $u_5, \mars$ &  $u_6, \venus$ & $u_7, \venus$  & $u_8, \venus$   \\
  0.97 & 0.93 & 0.89 & 0.81 & 0.73 & 0.72 & 0.64 & 0.62 \\
  \bottomrule
\end{tabular}
  \vspace{-3mm}
\end{table}
\begin{example}\label{ex:1}
Consider the case described in Table \ref{tab:running_example}
and suppose that the group-fairness constraint requires to have at least $\floor*{k/2}$ individuals of each gender in the top-$k$ positions starting from $k \geq 3$.

The ranking by decreasing relevance $r^* = \langle u_1, u_2, \ldots, u_8\rangle$ is not a valid ranking in this case, as $\venus$ is underrepresented in the top-$k$ positions for $k = 4, 5, 6$.
A valid ranking which is as close as possible to $r^*$ would be ${r'}  = \langle u_1, u_2, u_3, u_6, u_4, u_7, u_5, u_8\rangle$.
\end{example}

This approach stems from an information retrieval standpoint: the set of items to be ranked is the result of a query, and as long as the given group-fairness constraint is satisfied,
it suffices for the application at hand to maximize the global utility.
While at first sight this setting might seem adequate to rank people, maximizing global utility provides no guarantee to individuals, who care little about global utility. 
In Example \ref{ex:1}, individuals $u_4$ and $u_5$ have been uniquely penalized from a \emph{meritocratic fairness} point of view:
They may accept the group-fairness constraints and agree with the fact that the produced ranking $r$ is as close as possible to $r^*$, but nevertheless feel
discriminated against, for being
the only ones in a worse position in $r$ than in  $r^*$ despite other solutions being possible. For example,
$\langle u_4, u_1, u_3, u_6, u_2, u_7, u_5, u_8\rangle$ is valid and more favourable to $u_4$.
In other words, while the use of group-fairness constraints
    is often desirable and may be required by law,
    certain individuals in a such a valid ranking might feel unfairly penalized, even when comparing only to individuals within the same group.
   As soon as a group-fairness constraint is enforced in
ranking problems, some individual-level unfairness is inevitably introduced\footnote{This situation resembles some cases in fair
    classification in which enforcing statistical parity constraints cause a form of
    unfairness from an individual viewpoint~\cite{fairness_awareness}.}.

  In this paper we study \emph{the problem of minimizing
the amount of individual unfairness introduced when enforcing a group-fairness constraint}.
While much of the literature for ranking attempts to maximize global utility, global quality metrics generally fail to adequately capture
the treatment of \emph{individuals}.
Thus, differently from the literature which tries to maximize the global utility, we adopt Rawls's theory of justice~\cite{theory_justice}, which advocates arranging social and financial
inequalities to the benefit of the worst-off.
Following this precept, a natural task is to
    find a ranking that, while satisfying the group-fairness constraint, maximizes
the utility of the least-advantaged individual:

\begin{equation}\label{eq:individual}
{r'}  \in \argmax_{r \in \mathcal{S}} \min_{u \in \users} V(r,u).
\end{equation}
Here $V(r,u)$ represents the value (utility) that placing $u$ at position $r(u)$ brings to the individual $u$, \emph{relative to $u$'s quality $R(u)$}.

In Section~\ref{sec:baseline} we provide an exact optimal solution for~\eqref{eq:individual}. This, however, is not the main focus of our paper. In fact, we can improve individual treatment even further through \emph{randomization}.

\spara{{Randomization for individual fairness.}}
{We next show how, by means of randomization, we can improve individual treatment over the best deterministic solution of~\eqref{eq:individual}. In particular, we show that there exists a \emph{probability distribution over valid rankings}, where the minimum expected value that any individual gets is higher than is possible with any single ranking.}
\begin{example}\label{ex:2}
Consider the value function $V(r,u) = r^*(u) - r(u)$, i.e., the difference between the meritocratic ranking by relevance and the ranking produced. This is positive for individuals who are in a better (lower-ranked) position in $r$ w.r.t. $r^*$ and negative for
others.  It is easy to see that the ranking ${r'} $ in Example \ref{ex:1}  maximizes the minimum value of $V(r,u)$: in fact in order to have 3 $\venus$ in the first 6 positions,
some $\mars$
         has to give up at least 2 positions w.r.t. $r^*$.
\end{example}

Even when optimizing for~\eqref{eq:individual},
{individual} $u_5$ in Example \ref{ex:2} might have concerns for being the one receiving the largest part of the burden of satisfying the group-fairness constraint. The only way to improve on this situation is to
introduce {randomization} into the process. This means producing a probability distribution over possible valid rankings instead
of a single deterministic ranking.
\begin{example}\label{ex:3}
Consider the same instance of Example \ref{ex:1}. The following distribution over four rankings $r_{1-4}$ maximizes the minimum expected value of $V(r,u) = r^*(u) - r(u)$ among all individuals in $\users$:

\begin{center}
\begin{tabular}{l}
  \hline
\toprule
$\Pr(\langle u_1, u_4, u_3, u_7, u_2, u_6, u_5, u_8\rangle) = 1/4$ \\
$\Pr(\langle u_2, u_1, u_3, u_6, u_4, u_8, u_7, u_5\rangle) = 1/2$ \\
$\Pr(\langle u_2, u_1, u_3, u_7, u_5, u_6, u_4, u_8\rangle) = 1/16$ \\
$\Pr(\langle u_5, u_1, u_3, u_7, u_2, u_6, u_4, u_8\rangle) = 3/16$ \\
\bottomrule
\end{tabular}
\end{center}

It is easy to check that, under this distribution, everyone has expected value at least $-0.75$ (which is achieved by the four $\mars$), while under the best deterministic solution (Example  \ref{ex:2})
    we had $V(r,u_5) = -2 < -0.75$.
\end{example}

While in Example \ref{ex:2} the burden required for ensuring the group-fairness constraint  was all on $u_4$ and $u_5$, in Example \ref{ex:3} it has been equally distributed
    among the four $\mars$.
Notice that all four rankings in the distribution above satisfy the group-fairness constraint in Example~\ref{ex:1}.
    However, by
        combining these four rankings probabilistically,
                  we have succeeded in achieving a higher minimum expected value than is possible
    via any single deterministic ranking. In fact, we have also minimized the disparity in the expected value
    that each individual receives:
        whereas requiring all expected values to be the same is not mathematically possible when satisfying group constraints, the solution
     above comes as close as possible by minimizing the maximum gap.
A complete problem definition formalizing these ideas is given in Section~\ref{sec:problem}.

\spara{{Implications and practical deployment.}}
In order to guarantee the maximum possible value is brought to each individual, in this paper we embrace randomization and produce a probability distribution over possible valid
rankings. 
This distributional fairness approach is very well suited for a search context in which the same query can be served many times for different users of a platform (e.g., headhunters searching for a specific type of professional on a career-oriented social networking platform such as LinkedIn or XING).
Notice also that
\emph{amortized fairness} in the sense of~\cite{BiegaGW18,SinghJ18} is an immediate application of this distributional approach: if there are several rankings to be made, we can
draw them independently from a fair distribution of rankings, so that
the empirical properties of the sample approach those of the fair distribution.

However, the usefulness of randomization extends to settings with a
single, non-repeated trial (as in, e.g., university admissions). In this case it is an essential tool to secure ``ex-ante'' (procedural)
individual fairness, i.e.,
fairness of the procedure by which the outcome is selected, as opposed to ``ex-post'' fairness, which is based on the final outcome
alone (see, e.g.,~\cite{bolton2005fair}).

Regarding implementation and transparency issues, notice that instead of treating the algorithm as a black box outputting a single ranking, one can make the entire
distribution public.
For instance, we can publish the distribution described in in Example \ref{ex:3} above,
letting all the individuals verify the expected value, as well as the fact that this distribution is optimal under the maxmin-fair
criterion (see Section~\ref{sec:problem}). Then one of the four rankings $r_{1-4}$ can be picked at random, via any fair and transparent lottery mechanism or coin-tossing protocol.
Moreover, our algorithms guarantee that the optimal distribution found is supported on a small (polynomial-size) set of rankings, even if the space of all valid rankings is exponential.

\spara{{Paper contributions and roadmap.}}
In the rest of this paper, following the randomized \emph{maxmin-fairness framework}~\cite{dgs2020}, we study how to efficiently and accurately compute this type of
distributions over the rankings satisfying a given set of group-fairness constraints. We achieve the following contributions:
\squishlist
\item We introduce the distributional maxmin-fair ranking framework and provide the formal problem statement (Section \ref{sec:maxmin_general}). We show that maxmin-fair ranking distribution maintains within-group meritocracy and, in certain cases, it
has the desirable properties of being generalized Lorenz-dominant, and minimizing social inequality (Section \ref{subsec:properties}).

\item Our main result is an exact polynomial-time algorithm to find maxmin-fair distributions of many problems, including ranking (Section~\ref{sec:algorithms}).
A quicker method to find  maxmin-fair distributions approximately is explained in Appendix~\ref{sec:apx_fairness}.

\item We also provide an exact optimal solution (Section \ref{sec:baseline}) for the deterministic version of the problem as in (\ref{eq:individual}). This is achieved by a means of a variant of Celis et al.~\cite{CelisSV18}. We use this as a baseline allowing us to quantify the advantage of probabilistic rankings over the optimal deterministic ranking. 

\item Our experiments on two real-world datasets confirm empirically the advantage of probabilistic rankings over deterministic rankings in terms of minimizing the inequality for the worst-off individuals (Section~\ref{sec:experiments}).

\squishend

To the best of our knowledge, this is the first work studying the problem of minimizing
the amount of individual unfairness introduced when enforcing group-fairness constraints in ranking. A {major} contribution is showing how randomization can be {a} key {tool} in reconciling individual and
group fairness: {we believe that this might hold for other problems, besides ranking}.

\section{Related work}
\label{sec:related}
        There are some works on algorithmic fairness focused on individual fairness, but none of them considers them in conjunction with
        group fairness.
Dwork et al~\cite{fairness_awareness} introduce a notion of individual fairness in classification problems.
Roughly speaking, their definition requires that all pairs of similar individuals should be treated similarly.
 This is impossible to satisfy with a deterministic classifier so, similarly to ours, their
 definition of fairness requires randomized algorithms. The individual similarity metric is assumed given, while  they base their notion of 'similar treatment'  on the difference
 between the probabilities of a favourable classification.
         Kearns et al. \cite{KearnsRW17} introduce the notion of \emph{meritocratic fairness} in the context of selecting a group of individuals from incomparable populations (with no
        group-fairness constraint). Their notion intuitively requires that less qualified candidates do not have a higher chance of getting selected than more qualified ones.
Another work focusing on individual fairness is that of Biega et al. \cite{BiegaGW18}, which aims at achieving \emph{equity-of-attention fairness} amortized across many rankings, by requiring that exposure be proportional to relevance.

Our previous work~\cite{dgs2020} presents a very general framework to deal with individual
fairness, based on randomized maxmin-fairness: the idea is to use a
distribution of solutions in order to maximize the expected value for the worst-off individual. 
In particular, \cite{dgs2020}  analyzes the case of unweighted matching with no group-fairness constraint: presents efficient algorithms and shows that these maxmin-fair matching distributions minimize inequality. 
While the techniques from~\cite{dgs2020} are combinatorial and can only deal with unrestricted matchings,
we greatly generalize the algorithmic results therein via convex optimization techniques, showing that for a wide class of problems (including weighted matching and ranking with constraints), a
maxmin-fair distribution may be found in polynomial time; we only require the existence of a \emph{weighted optimization oracle} (see Section~\ref{sec:algorithms}).

The bulk of recent literature on fairness in ranking \cite{YangS17,CelisSV18,ZehlikeB0HMB17,SinghJ18,YangGS19,AsudehJS019,GeyikAK19} and learning-to-rank \cite{SinghJ19,DworkKRRY19,Narasimhan2020} deals with group fairness.
Singh and Joachims \cite{SinghJ18} propose an algorithm 
computing a fair
probabilistic ranking maximizing expected global utility. The fairness constraints expressible in their framework apply to the ranking distribution and not to each single ranking, as required by the group-fairness constraints we use. Celis et
al.~\cite{CelisSV18} also investigate fair ranking with group-fairness constraints with an objective function of the form~\eqref{eq:global}, assuming the values $W(r, u)$ satisfy the Monge condition. They give a 
polynomial-time algorithm for 
disjoint protected groups, and a faster greedy algorithm that works
 when only upper bound constraints are given.
 When the protected groups are allowed to overlap the problem becomes \NPhard\ and a polynomial-time approximation algorithm is provided
 in~\cite{CelisSV18}.

\section{Maxmin-Fair Ranking}
\label{sec:problem}

We are given a set of $n$ individuals to be ranked $\users$, a partition of $\users$ into groups $C_1, \ldots, C_t$, and a relevance function $R:\calU \to \reals$. For the sake of
simplicity we assume that ties are broken so that all $R(u)$ are distinct. Moreover, we are given group-fairness constraints as in~\cite{CelisSV18}, defined  by setting, for each
$i \in [n]$ and $k \in [t]$, a lower bound $l_i^k \in \naturals $ and an upper bound $u_i^k \in \naturals$ on the number of individuals from class $k$ in the first $i$ positions.
We denote by  $\mathcal{R}$ the set of all possible rankings
    of $\users$ (bijections from $\calU$ to $[n]$), and by $\sols \subseteq \calR$ the set of all valid rankings: \begin{equation}\label{eq:constraints}
\sols = \left\{ r \in \calR \mid l_i^k \le |\{ u \in C_k \mid r(u) \le i \}| \le u_i^k \ \; \forall i \in [n], k \in [t] \right\}.
\end{equation} Finally, we consider a value function $V:\sols \times \users \to \reals$
 such that $V(r,u)$ represents the value (utility) that placing $u$ at position $r(u)$ brings to the individual $u$, relative to $u$'s quality $R(u)$.
As we are interested in modeling meritocratic fairness, our value function must take into consideration the input relevance score $R(u)$ and the produced ranking $r(u)$.
We
consider value functions of the form:
                  \begin{equation}\label{eq:utility} V(r, u) =  f(r(u)) - g(u), \end{equation}
where $f: [n] \rightarrow \reals$ is
a decreasing function
and
 $g: \users \rightarrow \reals$ is
increasing in
$R(u)$.

The intuition is the following:
suppose that being assigned at position $i$ carries
intrinsic
                utility $f(i)$, while $u$'s merit for the ranking problem is $g(u)$ (which may depend on $u$ and hence also on $R(u)$); then $V(r, u)$  measures the net difference between $f(r(u))$ and $g(u)$, i.e., how
                much $u$ has gained in $r$ w.r.t. $u$'s actual merit.
In typical applications we can take any decreasing function $p:[n] \to \reals^{\ge 0}$ encoding \emph{position bias} or \emph{exposure} (see~\cite{position_bias} for common
 models) and set $f = p$ and $g =p \circ r^*$. As simple examples,  by setting $p(i) = n-i$ and $p(i) = \log(n/i)$, we can get $V(r, u) = r^*(u) - r(u)$ and $V(r, u) = \log ( \frac{r^*(u)}{r(u)})$.
 When the ranking is a selection process where
 $k \in \naturals$ individuals are selected and there is no advantage to being ranked first over $k_{th}$ as long as one is selected, we may use
 $$V(r,u) =\begin{cases}
                              1, & \text{if}\  r^*(u)   > k \text{ and } r(u) \le k \\
                              -1, & \text{if}\ r^*(u) \le k \text{ and } r(u) > k \\
                              0, & \text{otherwise}.
                            \end{cases}$$
These are but a few examples. Determining which value function $V$ is best from a psychological or economical standpoint is beyond the scope of this work.
Instead we  take $V$ as given and design algorithms which can efficiently deal with any
function of the form~\eqref{eq:utility}.

\subsection{Maxmin-fairness framework}\label{sec:maxmin_general}
Consider an input instance $\instance$ of a general search problem
which defines implicitly a set $\calS  = \calS(\instance)$ of feasible solutions, assumed to be finite and non-empty.
Let $\users$ denote a finite set of individuals and let us associate with each solution $S \in \sols$ and each individual $u \in \users$ a real-valued satisfaction $A(S, u) \in \reals$
(which in~\cite{dgs2020} takes binary values).
Consider
a \emph{randomized algorithm} $\calA$ that, for any given problem instance
$\instance$,  always halts and selects a solution $\calA(\instance) \in \sols$.
Then $\calA$ induces a probability distribution $D$ over $\sols$: $\Pr_D[S] = \Pr[\calA(\instance) = S]$
 $\forall S \in \sols$.
 Denote the \emph{expected satisfaction} of each $u \in \users$  under  $D$ by
$D[u] \triangleq \expect_{S \sim D}[  A(S, u) ].$
A distribution $F$ over $\sols$ is \emph{maxmin-fair} for $(\users, A)$ if it is impossible to improve the expected satisfaction of any individual without decreasing it for some other individual which is no better off, i.e.,
if for all distributions $D$ over $\sols$ and all $u \in \users$,
\begin{align}\label{eq:maxmin}
\cov{D}{u} &> \cov{F}{u} \implies \exists v \in \users \mid
    \cov{D}{v} < \cov{F}{v} \le \cov{F}{u}.
    \end{align}

Maxmin-fair distributions always exist~\cite{dgs2020}.  Due to the convexity of the set of feasible probability distributions, an
equivalent definition 
    can be given based on the sorted vectors of expected satisfactions.
Given a distribution $D$ over $\calS$, let $\scov{D} = (\lambda_1, \ldots, \lambda_n)$ be the
vector of expected satisfactions $(D[u])_{u \in \calU}$ sorted in increasing order.
Let $\succ$ denote the lexicographical order of vectors: i.e., $(v_1, \ldots, v_n) \succ (w_1, \ldots, w_n)$ iff there is some index $i\in[n]$ such that $v_i > w_i$ and $v_j =
w_j$ for all $j < i$. Write $v \succeq w$ if $v = w$ or $v \succ w$.
Then a distribution $F$ over $\sols$ is \emph{maxmin-fair} if and only if $\scov{F} \succeq \scov{D}$ for all distributions $D$ over $\calS$ \cite{dgs2020}.

\begin{problem}[Maxmin-fairness in combinatorial search]\label{prob:fair_general}
Given a fixed search problem, a set $\calU$ of individuals, and a satisfaction function $A$,
design a randomized algorithm $\mathcal A$ which always terminates and such that, for each instance  $\instance$,
 the distribution of $\calA(\instance)$ is maxmin-fair for $(\calU, A)$.
\end{problem}
Problem~\ref{prob:fair_general} is a general formulation of maxmin-fairness in search problems. Different choices for the set of feasible solutions $\sols$ and
the satisfaction function $A$ lead to different algorithmic problems. The problem involves
continuous optimization over infinitely many distributions, each defined over the set $\sols$ of valid solutions
(which is exponential-size).
Despite these difficulties,  we will
show that Problem~\ref{prob:fair_general} is tractable under mild conditions (Section~\ref{sec:algorithms}).

Garc\'ia--Soriano and Bonchi~\cite{dgs2020} instantiate Problem~\ref{prob:fair_general} with the case of matching. The main problem studied in the rest of this paper is obtained
by instantiating Problem~\ref{prob:fair_general} with the case of \emph{ranking under  group-fairness constraints with
individual-level value function}: in our setting $\sols$ is the set of rankings $r$ over $\users$ satisfying the group-fairness constraints, and $A(S,u)$ is our value function $V(r,u)$.

\begin{problem}[Maxmin-fair ranking with group-fairness constraints]\label{prob:main}
Given a set of individuals to be ranked $\users$, a partition of $\users$ into groups, a set $\sols$ of rankings satisfying a given set of group-fairness constraints as defined
in~\eqref{eq:constraints},  and a value function $V$ as defined in~\eqref{eq:utility}, design a randomized algorithm which outputs rankings in $\sols$,  such that its output distribution over $\sols$  is maxmin-fair.
\end{problem}

\subsection{Properties of maxmin-fair rankings} \label{subsec:properties}

We next state some important properties of maxmin-fair rankings. For the sake of readability, the proofs can be found in the Appendix. The first property states the maintenance of the meritocratic order within each group  of individuals (e.g., gender).

\begin{theorem}[Intra-group meritocracy]\label{thm:meritocracy}
For any two individuals $u_1,u_2 \in \users$ belonging to the same group and such that $R(u_1) \ge R(u_2)$, it holds that, if a distribution $F$ over valid rankings $\calS$ is
maxmin-fair, then
    $\expect_{r \sim F}[f(r(u_1))] \ge \expect_{r \sim F}[f(r(u_2))]$.
\end{theorem}

Our  second property employs the notion of \emph{(generalized) Lorenz dominance} from~\cite{shorrocks}, a
property indicating a superior  distribution of
net incomes.
Consider two ranking distributions $A$ and $B$. Let $A_{(k)} = \scov{A}[k]$ denote the $k_{th}$ element of the expected satisfaction values sorted in increasing order.
Then
\emph{$A$ dominates $B$} if
         $ \sum_{i=1}^k A_{(i)} \ge \sum_{i=1}^k B_{(i)} \; \forall k \in [n], $
i.e., the expected cumulative satisfaction of the bottom individuals is always higher or equal in~$A$.

A distribution is \emph{generalized Lorenz-dominant} if it dominates every other distribution.
When it exists, such a distribution has a strong claim to being superior to all others, in terms of equity and efficiency~\cite{shorrocks,thistle1989ranking,lambert1992distribution}. A generalized Lorenz-dominant
distribution must also be maxmin-fair. 
 We show that
 a dominant distribution
does exist for rankings, 
    in the important case where only upper bound constraints are given in~\eqref{eq:constraints}.
Notice that in the case of two groups (e.g., a binary protected attribute), lower bound constraints may be replaced with an equivalent set of upper bound constraints.

\begin{theorem}\label{thm:lorenz}
The maxmin-fair ranking with upper bounds is generalized Lorenz-dominant.
\end{theorem}

Since $\sum_{u \in \calU} V(r, u)$ is a constant independent of $r$, an easy consequence of Theorem~\ref{thm:lorenz} is that the maxmin-fair distribution also \emph{minimizes social inequality} in the sense
of~\cite{dgs2020}: i.e., the
maximum difference between the expected satisfactions of two users.
\begin{corollary}
The maxmin-fair ranking with upper bounds minimizes
$\max_{u\in \calU} D[u] - \min_{v \in \calU}D[v]$
over all ranking distributions $D$, as well as any other quantile range.
\end{corollary}
Moreover, by the majorization inequality~\cite{karamata}, it also maximizes any social welfare function that is additively separable, concave and symmetric w.r.t. $\users$:
\begin{corollary}
Suppose $f:\reals \to \reals$ is concave.
When only upper bound constraints are present, the maxmin-fair distribution maximizes
$ \sum_{u \in \calU} f( D[u] ) $
over all ranking distributions $D$.
\end{corollary}
In particular, in this case the maxmin-fair distribution minimizes the variance of $\scov{D}$, and when the values $V(r, u)$ are positive, it also maximizes, for instance, the Nash social
welfare~\cite{nash1950bargaining} (geometric mean)
 of
expected satisfactions. It must also minimize the Gini inequality index when it is well-defined~\cite{shorrocks}.

\section{Algorithm}
\label{sec:algorithms}
We show that our fair ranking problem (Problem~\ref{prob:main}) is efficiently solvable.
Notice that the set $\sols$ of valid solutions can be exponential-size, so enumerating  $\sols$  is out of the question in an efficient algorithm. Instead, we need a method to quickly single out the best solutions to combine for a maxmin-fair distribution.
To show how this can be done, we abstract away from the specifics of the problem and show how to
find  a maxmin-fair distributions of general search problems (Problem~\ref{prob:fair_general}).
The following notion  is key:
\begin{definition}
A \emph{weighted optimization oracle} for $A: \sols \times \users \to \reals$ is an algorithm that, given $w: \calU \to \reals^{\ge 0}$,
returns $S^*(w)$ and $A(S^*(w), u)$ for all $u \in \calU$, where
\begin{equation}\label{eq:oracle}
S^*(w) \in \argmax_{S \in \sols} \sum_{u \in \calU} w(u) \cdot A(S, u).
\end{equation}
\end{definition}
Roughly speaking, the intuition why these oracles are important is the following.  Suppose we have constructed a distribution $D$ which is not maxmin-fair.
By putting more weight on the individuals less satisfied by $D$, we can use
the weighted optimization
oracle to find a new solution $S$ placing more emphasis on them, which can be added
to ``push $D$ towards maxmin-fairness''.

Designing an efficient weighted optimization oracle is a problem-dependent task. Our first algorithmic result  reveals that their existence
suffices to solve Problem~\ref{prob:fair_general} efficiently.

\begin{theorem}\label{thm:general_algo}
Given
a weighted optimization oracle, Problem~\ref{prob:fair_general} is solvable in polynomial time.
\end{theorem}
We emphasize that Theorem~\ref{thm:general_algo} is very general and its  applicability
is in no way limited
to ranking problems, or to value
functions of a certain form. They apply to an arbitrary search problem $P$ (e.g., searching for a ranking, a matching, a clustering...) and an arbitrary set of
individuals. As long as
an efficient weighted optimization algorithm exists for $P$,
    it yields efficient algorithms for maxmin-fair-$P$ (Problem~\ref{prob:fair_general}).
The wide applicability of this condition implies that maxmin-fair distributions may be efficiently solved in a great many cases of interest: most polynomial-time solvable problems studied in combinatorial optimization (e.g., shortest paths, matchings, polymatroid intersection\dots) admit a polynomial-time weighted optimization oracle.
Thus, Theorem~\ref{thm:general_algo} extends in a new direction the results
of~\cite{dgs2020}:
as efficient weighted matching algorithms exist, the main result of~\cite{dgs2020} becomes a corollary to Theorem~\ref{thm:general_algo} (up to the loss of a polynomial factor in runtime).

More importantly for us, the same holds for constrained ranking.
\begin{theorem}\label{thm:monge} Ranking with group-fairness constraints as in~\eqref{eq:constraints} and a value function of the form~\eqref{eq:utility} admits a
polynomial-time weighted optimization oracle.  \end{theorem}
\begin{corollary}
Maxmin-fair ranking with group-fairness constraints (Problem~\ref{prob:main}) is solvable in polynomial-time.
\end{corollary}

In Section~\ref{sec:proof_exact} we prove Theorem~\ref{thm:general_algo} by solving a sequence suitably designed linear programs. Each of these programs requires exponentially many constraints to be written down explicitly, but can nonetheless be solved efficiently via
the ellipsoid method using a weighted optimization oracle. (As explained in Appendix~\ref{sec:apx_fairness}, if we settle for some approximation error, these LPs can also be solved approximately using techniques to solve zero-sum games and packing/covering
LPs~\cite{young1995randomized,freund1999adaptive,packing_covering}.)
Finally, in Section~\ref{sec:monge} we show the existence of weighted optimization oracles for ranking (Theorem ~\ref{thm:monge}).

\subsection{Proof of Theorem~\ref{thm:general_algo}}\label{sec:proof_exact}
       We start by showing a weaker result concerning the computation of the optimal expected satisfaction values, rather than the actual distribution of solutions.
\begin{lemma}\label{thm:fairness_P}
    Given a weighted optimization oracle,
    the expected satisfactions of a maxmin-fair
distribution can be computed in polynomial time.
\end{lemma}
\vspace{-0.1cm}
\begin{proof}
Let $F$ be a maxmin-fair distribution. We maintain the invariant that we know the expected satisfaction $\alpha_v$ of $F$ for all $v$ in a subset $K \subseteq \users$:
\begin{equation}\label{eq:know}
 F[v] = \alpha_v \text{ for all $v\in K$ },
   \end{equation}
\begin{equation}\label{eq:inorder}
K \neq \emptyset \implies \cov{F}{v} \ge \max_{w \in K} \alpha_w \text{ for all $v \notin K$}.
\end{equation}

Initially $K = \emptyset$. We show how to augment $K$ in polynomial time while maintaining~\eqref{eq:know} and~\eqref{eq:inorder}, which gives the result since $K = \calU$ will be reached after at most $|\calU|$ iterations.
We need to find the largest minimum expected satisfaction possible outside $K$ for a distribution $D$ subject to the constraints that the expected satisfaction inside $K$ must be equal to  $\alpha_v$.
By~\eqref{eq:know}, ~\eqref{eq:inorder} and the lexicographical definition of maxmin-fairness, for any distribution $D$ the constraints $D[v] = \alpha_v$ for all $v \in K$ are equivalent to the
constraints $D[v] \ge \alpha_v$ for all $v
\in K$. We can write our optimization problem as the following (primal) linear program:
    \begin{equation}\label{lp:fair3}
    \begin{array}{rrclcl}
    \displaystyle \max & \lambda \\
            \\
    \textrm{s.t.}
    & \displaystyle \sum_{S \ni v} -{p_S}\cdot A(S, v) & \le & -\alpha_v & \forall v \in K \\
    & \displaystyle \lambda + \sum_{S \ni v} -{p_S}\cdot A(S, v)  & \le & 0 & \forall v \notin K \\
    & \displaystyle \sum_{S \in \mathcal{S}} p_S  & = & 1 &  \\
    & p_S    &\ge& 0, \\
    \end{array}
    \end{equation}
whose dual is
    \begin{equation}\label{lp:zv3}
    \begin{array}{rrclcl}
    \displaystyle \min & \mu
        - \displaystyle \sum_{v \in K} \alpha_v w_v \\
    \textrm{s.t.}
    & \displaystyle \mu - \sum_{v \in S} w_v \cdot A(S, v)  & \ge& 0 & \forall S \in \sols \\
    & \displaystyle \sum_{v \in \users} w_v  & = & 1 &  \\
    & w_v &\ge& 0. \\
    \end{array}
    \end{equation}
The dual~\eqref{lp:zv3} has $|\users|$ variables but a
possibly exponential number of constraints (one for each candidate solution $S$).
 To get around this
difficulty, observe that it
can be written in the equivalent form
                    {
                        \begin{equation}\label{lp:zv4}
                        \begin{array}{rrclcl}
                        \displaystyle \min & \displaystyle \left[ \max_{S\in\mathcal{S}} \displaystyle \sum_{v \in S} w_v \cdot A(S, v) \right] - \sum_{v \in K} \alpha_v w_v \\
                        \textrm{s.t.}
                        & \displaystyle \sum_{v \in \users} w_v  & = & 1 &  \\
                        & w_v &\ge& 0. \\
                        \end{array}
                        \end{equation}
                    }%
This formulation makes it apparent that, given a weighted optimization oracle, we can construct a \emph{separation oracle} for the dual, i.e., an algorithm that given a candidate solution
to~\eqref{lp:zv3} and a parameter $\lambda$, returns ``yes'' if it is a feasible solution of value at most $\lambda$, and otherwise returns ``no'' along with some violated
constraint or reports the fact that the value of the candidate solution is larger than $\lambda$. Indeed, given $\{w_v\}_{v \in \users}$ and
$\lambda \in \reals$, we can determine if the optimum of~\eqref{lp:zv4} is at most $\lambda$ by using the weight
optimization oracle and answering
``yes'' if the weight of the solution is no larger than $\lambda + \sum_{v\in K} \alpha_v w_v$. Otherwise the separation oracle answers ``no'' and
reports a violated constraint, given either by the
constraint $\sum_v w_v = 1$, which can be checked separately, or by the constraint
    $ \sum_{v \in S^*} w_v \cdot A(S^*, v)  \le \mu, $
where $S^*$ is the solution found by the weighted optimization oracle.

The existence of a separation oracle for a linear program~
implies its polynomial-time solvability via the ellipsoid algorithm~\cite{separation_oracle}.
Hence~\eqref{lp:zv4} can be solved exactly in polynomial time, and we can find an optimal solution to~\eqref{lp:zv4}. Let us denote the optimal primal and dual solutions by $\{p^*_S\}_{S \in \sols}$ and $\{w_v^*\}_{v \in \users}$.
Suppose now that the optimum value of~\eqref{lp:fair3} is $\lambda^*$;
notice that if $K \neq \emptyset$, we must have $\lambda^* \ge \max_{v \in K} \alpha_v$ by our assumptions~\eqref{eq:know} and~\eqref{eq:inorder}.
Let $K' = \support(w^*) \setminus K = \{ v \notin K \mid
w^*_v > 0\}$.
By complementary slackness, for every $v \in K'$ its corresponding primal constraint in~\eqref{lp:fair3} is tight, hence $\sum_{S \ni v} p^*_S A(S, v) = \lambda^*$.
From the lexicographical definition of maxmin-fairness we infer that $F[v] = \lambda^*$ for all $v \in K'$ and $F[v] \ge \lambda^*$ for all $v \notin
K$. Therefore adding $K'$ to $K$ maintains the invariants~\eqref{eq:know} and~\eqref{eq:inorder} if we set $\alpha_v = \lambda^*$ for all $v \in K'$.
This allows us to augment $K$ as long as $K' \neq \emptyset$.

On the other hand, if $K' = \emptyset$, then $K \neq \emptyset$ (as $\sum_v w_v= 1$) 
and $\lambda^*= \max_{v \in K} \alpha_v$, since the objective function did not increase since the last iteration.
In this case we simply add the constraint
$\mu - \sum_{v \in K} \alpha_v w_v = \lambda^*$ to ~\eqref{lp:zv3} 
and change the objective function to minimize $\sum_{v \in K} w_v$. This also yields an optimal solution to
~\eqref{lp:zv4}. But in this case the new solution $w^{**}$ must satisfy $\support(w^{**}) \setminus K \neq \emptyset$, so we are back to the previous case.

We repeat this process until $K = \users$. The number of iterations is at most $|\users|$, and each iteration runs in polynomial time.
\end{proof}

\begin{proofof}{Theorem~\ref{thm:general_algo}}
Consider the last pair of LP programs used in the proof of
Lemma~\ref{thm:fairness_P} (i.e., when $K \cup K'=\users$). We used the
separation oracle and the ellipsoid algorithm to solve the dual LP~\eqref{lp:zv4}; it remains to show that we can also find a solution to the primal problem~\eqref{lp:fair3}, whose
variables
define the maxmin-fair distribution.
If all numbers $A(S, u)$ are rationals whose numerators and denominators are specified with with $b$ bits,
    then
    the number $T$ of calls to the separation
oracle during the run of the ellipsoid algorithm can be bounded by a polynomial in $|\users|$ and $b$
(see~\cite{separation_oracle}). Consider the subprogram $P$ of~\eqref{lp:zv4} formed by using only these
$T$ constraints, along with  $\sum_v w_v = 1$ and the non-negativity constraints. If we run the
ellipsoid algorithm (which is deterministic)  on the new subprogram $P$ instead, we will find the same solution, because the separation oracle will return the exact same sequence of
solutions.  Since the ellipsoid algorithm is guaranteed to find an optimal solution, it follows that the
reduced set of constraints is enough by itself to guarantee that the optimum of LP~\eqref{lp:zv4} is at least
$\lambda^*$ (hence exactly $\lambda^*$); all other constraints in~\eqref{lp:zv4} are redundant (their inclusion does not further increase the minimum objective value). The dual of this subprogram $P$ is a subprogram $Q$ of the primal
LP~\eqref{lp:fair3} using only $T$ of the variables $p_S$ and having the same optimal value as~\eqref{lp:zv4} and~\eqref{lp:fair3}. Since $Q$ has a polynomial number of variables and constraints, it can be solved
explicitly in polynomial time; any
solution to this reduced primal subprogram  $Q$ gives the desired distribution.
\end{proofof}

Pseudocode for the maxmin-fair algorithm is given below.

\begin{algorithm}
\small{
\DontPrintSemicolon
\caption{Maxmin-fair solver}\label{alg:exact}

\SetKwFunction{oracle}{oracle}
\SetKwInOut{Input}{input}
\SetKwInOut{Output}{output}

\Input{User set $\users$; weighted optimization $\oracle_{A}$ for $A\colon \sols \times \users \to \reals$}
\Output{A maxmin-fair distribution for $\calS$}
$K \gets \emptyset$

$\alpha_u \gets -\infty$ for all $u \in \users$

\While {$K \neq \users$}{
    Solve~\eqref{lp:zv4} using $\oracle_A$ as separation oracle.

    Let $\lambda^*$ be the optimal value.

    Let $\{w_v^*\}_{v \in \users}$ be a solution with value $\lambda^*$ and
    $\support(w^*) \setminus K \neq \emptyset$.

    $K' = \support(w^*)\setminus K = \{ v \notin K \mid w^*_v > 0\}$.

    $\alpha_v \gets \lambda^*$ for all $v \in K'$

    $K = K \cup K'$
}
$V \gets $ violated constraints found by the separation oracle in previous calls.

Solve the subprogram $Q$ of~\eqref{lp:fair3} comprising the constraints in $V$ and the simplex constraints $p_S \ge 0, \sum_S p_S= 1$.

\Return an optimal solution  $\{p^*_S\}_{S \in V}$ to $Q$.
}
\end{algorithm}

\subsection{Proof of Theorem \ref{thm:monge} }\label{sec:monge}
\begin{proof}
      Sort $\calU = \{u_1, \ldots, u_n\}$ by decreasing order of $w$
so that
\begin{equation}\label{eq:good_order}
w(u_1) \ge w(u_2) \ge \ldots \ge w(u_n),
\end{equation}
and let us identify $\calU$ with the set $[n]$ for ease of notation, so that $u_i = i$.
Recall that the positions $[n]$ are sorted by decreasing $f$:
\begin{equation}\label{eq:good_order2}
f(1) \ge f(2) \ge \ldots \ge f(n).
\end{equation}
Define $B(i, u) = f(i) - g(u)$ and $W_{iu} = w(u) \cdot B(i, u)$. Observe that, because of the orderings defined by~\eqref{eq:good_order} and~\eqref{eq:good_order2}, the matrix $W$ satisfies the following ``Monge property'':
if $i < j$ and $u < v$, then $W_{iu} + W_{jv} \ge W_{iv} + W_{ju}$. Indeed,
\begin{align*}
    W_{iu} + W_{jv} &- ( W_{iv} + W_{ju} ) = w(u) (B(i, u) - B(j, u))\\ &+ w(v) (B(j, v) - B(i, v))\\
     &= \left({w(u)} - {w(v)}\right) (f(i) - f(j))
     \ge 0.
\end{align*}
Thus we may apply the algorithm{%
\footnote{In~\cite{CelisSV18} an additional monotonicity property is assumed (that $W_{iu}$ is decreasing with $u$), but it is easy to check that it is not actually needed.}
} 
    from~\cite{CelisSV18} to find a valid ranking $r$ maximizing $\sum_{u \in \users} W_{r(u), u} = \sum_{u \in \users} w(u) \cdot V(r, u),$  as  required by the definition of weighted
    optimization oracle from Section~\ref{sec:algorithms}.
    Then we may compute each $V(r, u)$
    explicitly using $f$ and $g$.
\end{proof}

An important case is where only upper bounds are given in the group constraints, i.e., when the set of valid rankings is of the form
\begin{equation}\label{eq:only_upper}
\sols = \left\{ r \in \calR \mid |\{ u \in C_k \mid r(u) \le i \}| \le u_i^k \ \; \forall i \in [n], k \in [t] \right\}.
\end{equation}
Plugging in the algorithm from~\cite{CelisSV18} into Theorem~\ref{thm:monge} we obtain:
\begin{algorithm}
\small{
\DontPrintSemicolon
\caption{Weighted optimization oracle for ranking with upper bounds}\label{alg:weighted_upper}

\SetKwFunction{oracle}{oracle}
\SetKwInOut{Input}{input}
\SetKwInOut{Output}{output}

\Input{Set of individuals $\users$; weight function $w: \users \rightarrow \reals$; value function $V:\sols \times \users \to \reals$}
\Output{Best response ranking $r$ and $V(r, u)$ for all $u \in \users$}
Sort individuals in $\calU$ in order of {decreasing} weight: $w(u_1) \ge w(u_2) \ge \ldots \ge w(u_n)$.

For each position $i \in [n]$ in increasing order (as in~\eqref{eq:good_order2}), let $u$ be the smallest-index unassigned individual whose additional placement at position $i$ does not violate the group upper bound constraints in the first $i$ positions, and set $r(u) = i$.

Return $r$ and $V(r, u)$ for all $u \in \users$.
}
\end{algorithm}
\vspace{-0.1cm}
\begin{corollary}\label{lem:greedy}
Algorithm~\ref{alg:weighted_upper} is a weighted optimization oracle for ranking with upper bounds.
\end{corollary}

\section{Deterministic Baseline}
\label{sec:baseline}
In this section we present an exact optimal solution for the deterministic version of the problem as formulated in
(\ref{eq:individual}). This is useful for our experiments (Section \ref{sec:experiments}) as it allows us  to quantify the
advantage of probabilistic rankings over deterministic rankings in terms of the amount of individual fairness maintained.

Although Celis et al.~\cite{CelisSV18} study the problem of the form~\eqref{eq:global}, we devise a variant of their algorithm to deal with the problem as in~\eqref{eq:individual}:
this variant can be shown to provide the \emph{optimal deterministic ranking} solution to the constrained ranking problem \eqref{eq:individual} when the group-fairness constraints
are expressed in terms of upper bounds on the number of elements from each class that appear in the top-$k$ positions, 
    as in~\eqref{eq:only_upper}.
As noted in Section \ref{sec:problem} and  in \cite{CelisSV18}, in the case of two disjoint groups (e.g., a binary protected attribute such as gender), lower bound constraints may be replaced with an equivalent set of upper bound constraints.

\begin{algorithm}
\small{
\DontPrintSemicolon
\caption{Deterministic baseline}\label{alg:baseline}

\SetKwFunction{oracle}{oracle}
\SetKwInOut{Input}{input}
\SetKwInOut{Output}{output}

\Input{Set of individuals $\users$; relevance function $R: \users \rightarrow \reals^{\ge 0}$
}
\Output{Deterministic ranking $r$.}
Sort individuals in $\calU$ in order of {decreasing} score: $R(u_1) \ge R(u_2) \ge \ldots \ge R(u_n)$.

For each position $i \in [n]$ in increasing order (as in~\eqref{eq:good_order2}), let $u$ be the smallest-index unassigned individual whose additional placement at position $i$ does not violate the group upper bound constraints in the first $i$ positions, and set $r(u) = i$.

Return $r$.
}
\end{algorithm}

At the basis of our deterministic baseline (Algorithm \ref{alg:baseline}) lies the idea of using
the function
$\softmin(x_1, \ldots, x_n) = -\ln (\sum_{i=1}^n e^{-x_i})$
to force the algorithm from~\cite{CelisSV18} to approximately maximize a minimum instead of a sum, and observe that the limiting behaviour of the function
$x \to \softmin(M x) / M$ must also occur in this case for finite $M$ because the algorithm from~\cite{CelisSV18} does not depend on the specific values of the matrix $W_{ij}$, but only on
the existence of an ordering of
rows/columns
of $W$ where the Monge property holds (see Section~\ref{sec:monge}).

\begin{theorem}\label{lem:greedy2}
When the group-fairness constraints are defined only by upper bounds, Algorithm \ref{alg:baseline} returns a ranking $r'$ 
such that
$$r'  \in \argmax_{r \in \mathcal{S}} \min_{u \in \users} V(r,u).$$
\end{theorem}

\section{Experiments}
\label{sec:experiments}

\spara{Datasets.} We use two real-world datasets containing
gender information and one score for each individual.
Our first dataset comes
from the IIT Joint Entrance Exam (known as IIT-JEE 2009)~\cite{CelisMV20}\footnote{\url{https://jumpshare.com/v/yRUSJrnw3bzGGNf0jL3A}},  from which we select the top $N=1000$ scoring males and the top $N$ scoring females. The score distribution is heavily biased at the top, with just four females making the top-100.
Our second dataset is much less skewed: it contains
admissions data from all of the public law schools in the United States\footnote{\url{http://www.seaphe.org/databases.php}}. We use the top
$N=1000$ LSAT scorers, of whom 362 are female.

\begin{table}[t!]
\centering
    \captionsetup{font=small}
\caption{Minimum expected value produced by \textsf{MF}(0) and optimal deterministic solution, spread (maximum - minimum) of expected value, Gini inequality index (\%), and discounted cumulative gain for \textsf{IIT-JEE} and \textsf{Law-schools} datasets for different values of $\alpha$. \label{tab:minsat}}
  \vspace{-2mm}
\begin{adjustbox}{width=1.05\linewidth}
\hspace{-8mm}\begin{tabular}{r|ccc|c|ccc|}
\multicolumn{1}{c}{} & \multicolumn{3}{c}{\textsf{IIT-JEE}} & \multicolumn{1}{c}{} & \multicolumn{3}{c}{\textsf{Law-schools}}\\
  \cline{2-4}   \cline{6-8}
\multicolumn{1}{c|}{} & $\alpha = 0.1$ & $\alpha = 0.2$ & $\alpha =0.3$ & \multicolumn{1}{c|}{} &  $\alpha = 0.1$ & $\alpha = 0.2$ & $\alpha =0.3$\\
  \cline{2-4}   \cline{6-8}
    $minV(MF)$        & -26.82 & -96.44 & -185.7 & \multicolumn{1}{c|}{} &  -0.87 &-1.03  & -5.48 \\
    $minV(det)$       & -44   &  -180   & -358 & \multicolumn{1}{c|}{} &  -1  & -2   & -10      \\   \cline{2-4}   \cline{6-8}
    $sprd(MF)$        & 53.76 & 193.4 & 372.8& \multicolumn{1}{c|}{} &  0.97 & 1.14 &  6.19  \\
    $sprd(det)$        & 433 & 899 & 1192  & \multicolumn{1}{c|}{} &  6  & 7  & 32   \\   \cline{2-4}   \cline{6-8}
    $gini(MF)$        &  0.6714 & 2.413 & 4.658 & \multicolumn{1}{c|}{} &  0.0005   & 0.0011 & 0.05 \\
    $gini(det)$       &  1.062  & 3.772 & 7.027 & \multicolumn{1}{c|}{} &  0.0010   & 0.0020  & 0.08 \\  \cline{2-4}   \cline{6-8}
    $DCG(MF)$           & 84847$\pm$118  & 84444$\pm$180 & 84062$\pm$242  & \multicolumn{1}{c|}{} &  31379$\pm$ 1 &  31379$\pm$ 1 & 31379$\pm$ 1  \\
    $DCG(det)$          & 85123         & 85008        & 84807      & \multicolumn{1}{c|}{} &  31380 & 31380 & 31380  \\
 \cline{2-4}   \cline{6-8}
\end{tabular}
\end{adjustbox}
\end{table}
\begin{figure}[t!]
    \captionsetup{font=small}
  \centering
 \hspace{-4mm}\includegraphics[trim=52mm 0mm 60mm 12mm, clip=true,width=1.05\linewidth]{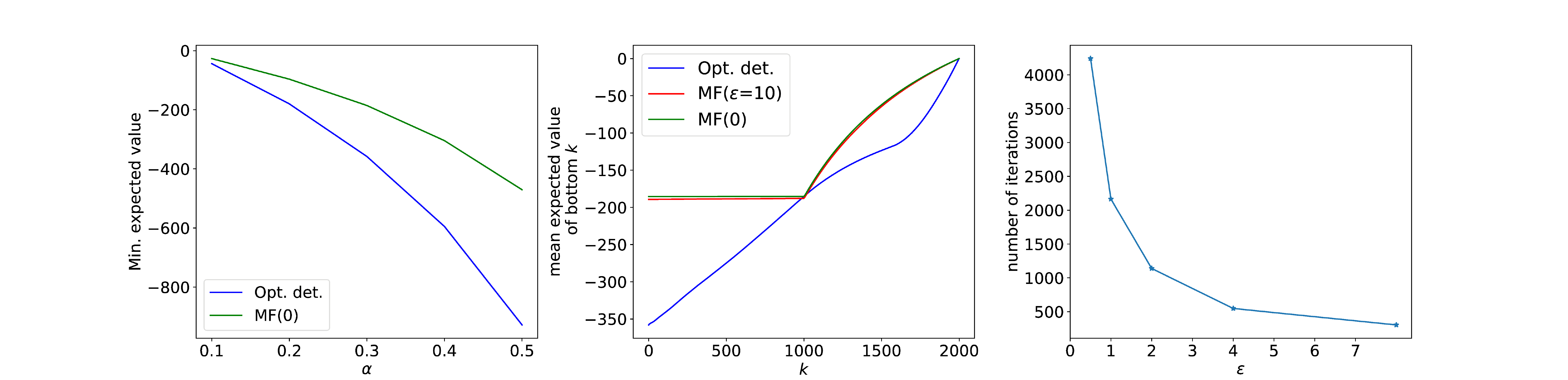}
  \caption{\textsf{IIT-JEE}: Minimum expected value produced by \textsf{MF}(0) and optimal deterministic solution (left); distribution of expected value $V(r,u)$ (for $\alpha = 0.3$)  (center); number of iterations (calls to the optimization oracle) vs error $\epsilon$ (right).}\label{fig:exp}
  \vspace{-4mm}
\end{figure}

\spara{Settings.} We impose the following group-fairness constraints, parameterized by $\alpha \in [0,\frac{1}{2}]$: at least $\lceil \alpha \cdot k - 1 \rceil$ females should be ranked in the top $k$, for $k=1,2,\ldots,2N$.
We employ $V(r,u) = r^*(u) - r(u)$ as our value function, where $r^*$ is the ranking by decreasing score.

\spara{Algorithms.} We implement our maxmin-fair solver for ranking, using the technique
described in Appendix~\ref{sec:apx_fairness} to solve the LPs approximately with an additive error parameter $\epsilon$; $\epsilon=1$ corresponds to an additive error in expected ranking
position of 1 (out of 2000 for IIT-JEE and out of 1000 for Law school).
We denote by $\textsf{MF}(\epsilon)$ the ranking distribution produced by
our approximate maxmin-fair algorithm with parameter $\epsilon$,
    and by $\textsf{MF(0)}$ the one obtained
with the smallest $\epsilon$ tested (0.5). 
Our code is available on Dropbox\footnote{\url{https://www.dropbox.com/sh/0kc17h36p632m0a/AACyO_PNPeBOJvPirEhQzFUDa?dl=0}}.

In order to quantify the advantage of probabilistic rankings over the optimal deterministic ranking, we also test the deterministic algorithm we devised (Algorithm~\ref{alg:baseline}) to solve the problem in~\eqref{eq:individual}.
This provides the strongest possible deterministic competitor for our algorithm.

\spara{Measures.}
Besides comparing the minimum expected value, which is the main focus of our work, we also report other measures of inequality of the produced solution:
spread (maximum - minimum) of expected value and Gini inequality index~\cite{gini1921measurement} (after normalizing values to the interval $[0,1]$ to make the index well-defined).
Finally, to examine if there is a loss in global ranking quality, we use the popular \emph{discounted cumulative gain}
metric
\cite{jarvelin2002cumulated,wang2013theoretical,CelisMV20,CelisSV18,SinghJ18,BiegaGW18},
    which can be defined as $DCG(r) = \sum_{u \in \users}  score(u) / \log( r(u) + 1)$.

\spara{Results.}
The first two rows of Table~\ref{tab:minsat} report the expected value (over a random ranking from the distribution) of the solution for the worst-off individual; we can observe that the maxmin-fair solution improves significantly on the optimal deterministic solution, with the gap between the two increasing with $\alpha$ (the strength of the group-fairness constraint). The same can be observed in Figure \ref{fig:exp} (left) and Figure \ref{fig:exp2} (left) for the two datasets.
We do not report the \emph{average} value of the solution for all individuals because it is the same for every ranking, as rankings are bijections onto
 $[n]$.

In Table~\ref{tab:minsat} we can also observe that the inequality measures for the maxmin-fair solution are always smaller than the optimal deterministic one.
Finally, we report the ranking-quality  measure DCG. Since, unlike the three other measures in Table~\ref{tab:minsat} DCG is
defined for deterministic rankings, we report average and standard deviation. We see that DCG is nearly the same for \textsf{MF}(0) and
det. Thus in this experiment \emph{ improving
individual fairness with respect to a group-only fairness solution incurs a negligible loss in DCG.}

Figure~\ref{fig:exp} (center) and Figure~\ref{fig:exp2} (center) depict the average expected value of the bottom $k$ individuals in three solutions: our best solution $\textsf{MF}(0)$, an approximate solution with $\epsilon=10$, and
the optimal deterministic solution. The peculiar behaviour  of the curve in Figure~\ref{fig:exp} (center)  (constant up to roughly $k=n/2$ for $\textsf{MF}$) is due to the skew of the input scores, which forces the
    maxmin-fair solution to essentially increase the ranking positions of most men by a certain minimum amount $X$ and decrease that of most women by $X$ with the best possible
        distribution. 
We notice that the maxmin-fair solution yields stronger cumulative value to the worst-off users than the other two do, for {any} $k$. In particular, the  maxmin-fair solution found Lorenz-dominates the approximate one and the deterministic one, in accordance with Theorem~\ref{thm:lorenz}.
Because of the error allowed, the approximate solution $\textsf{MF}(10)$ stays somewhat below $\textsf{MF}(0)$ and its curve crosses that of the deterministic solution sporadically before distancing itself again. Finally,
Figure~\ref{fig:exp} (right) and Figure~\ref{fig:exp2} (right) show the number of calls to the optimization oracle (which is also the size of the support of the ranking distribution) as a function of the additive error parameter $\epsilon$.
Runtime is linear in the number of calls to the optimization oracle. The longest runtime of our Python implementation of MF (which occurred on the IIT-JEE dataset with $\alpha = 0.3$ and $\epsilon = 0.5$) was
under one hour.

\begin{figure}[t!]
    \captionsetup{font=small}
  \centering
\hspace{-5mm}\includegraphics[width=1.05\linewidth]{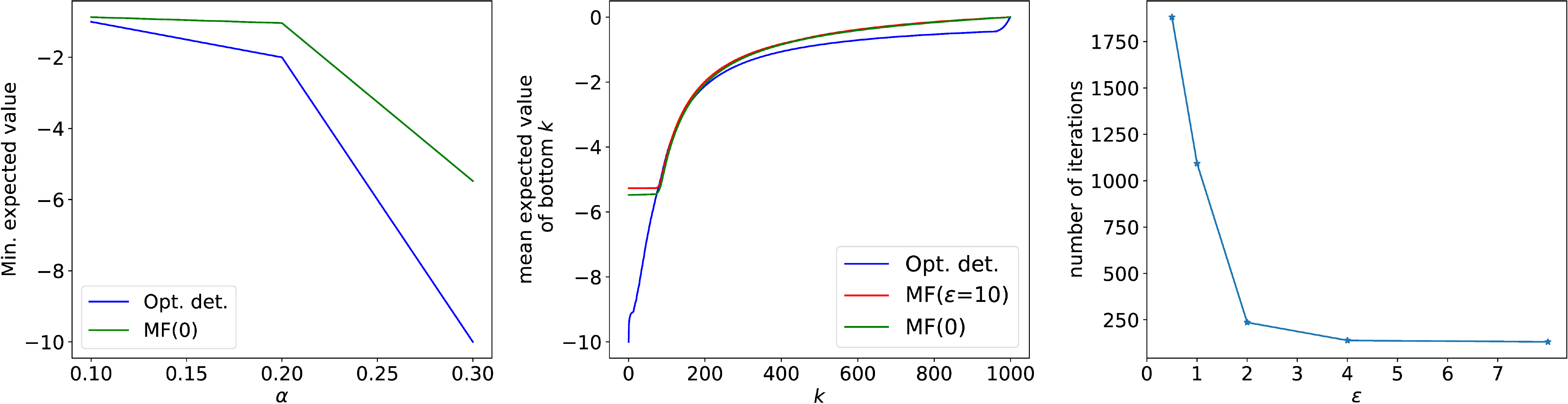}
  \caption{\textsf{Law-schools}: \ Minimum expected value produced by \textsf{MF}(0) and optimal deterministic solution (left); distribution of expected value $V(r,u)$ (for $\alpha = 0.3$)  (center); number of iterations (calls to the optimization oracle) vs error $\epsilon$ (right).}\label{fig:exp2}
  \vspace{-4mm}
\end{figure}

\section{Conclusions} \label{sec:conclusions}
We introduced the problem of minimizing the amount of individual unfairness introduced when enforcing group-fairness constraints in ranking. We showed how a randomized approach ensures more individual fairness than the optimal solution to the deterministic formulation of the problem.
We proved that our maxmin-fair ranking distributions provide strong fairness guarantees such as maintaining within-group meritocracy and, under a mild assumption (i.e., when we
        have only upper-bound constraints or when the protected attribute is binary), they have the desirable properties of being
generalized Lorenz-dominant, and minimizing social inequality. Besides the technical contributions, our work shows how randomization is key in reconciling individual and group fairness. In our future work we plan to extend this intuition beyond ranking.

\section*{Acknowledgements}
The authors acknowledge support from Intesa Sanpaolo Innovation Center.
The funders had no role in study design, data collection and analysis, decision to publish, or preparation of the manuscript.

\bibliographystyle{ACM-Reference-Format}
\bibliography{references}

\appendix
\section{Appendix}
We present here the proofs missing from the main text.

\subsection{Proof of Theorem~\ref{thm:meritocracy}}
\begin{proof}
Recall the form of our value function $V(r, u) = f(r(u)) - g(u)$ and
 observe that for any distribution $D$,
\begin{equation}\label{eq:blah2}
D[u] = \expect_{r \sim D} [V(r, u)] = \expect_{r \sim D}[ f(r(u))] - g(u).
\end{equation}

Let us write $r \in F$ to mean that $r$ occurs with non-zero probability in the maxmin-fair distribution $F$.
We show that if $u_1, u_2$ belong to the same group and $R(u_1) \ge R(u_2)$, then the following holds:
\begin{equation}\label{eq:oneoranother}
\text{$ f(r(u_2)) > f(r(u_1))$ for some $r\in F \implies F[u_1] \ge F[u_2]$ },  
\end{equation}
\begin{equation}\label{eq:blah}
\expect_{r \in F} [f(r(u_1))] \ge \expect_{r \in F} [f(r(u_2))].
\end{equation}
Suppose by contradiction that~\eqref{eq:oneoranother} fails, so
$ f(r(u_2)) > f(r(u_1))$ but $F[u_1] < F[u_2]$ for some $r \in F$.
Let $\hat{r}$ denote a ranking which is identical to $r$ except that the positions of $u_1$ and
$u_2$ are swapped.
As $u_1$ and $u_2$ belong to the same group, swapping their positions will not affect the group-fairness constraints, so  $\hat{r}$ is a valid ranking too. 
Consider a distribution $D$ over valid rankings $\calS$ obtained by drawing $s$ from $D$ and returning $s$ if $s \neq r$ and $\hat{r}$ if $s = r$.
We have
$ \expect_{s \in D} [ f(s(u_1)) ] - \expect_{s \in F} [ f(s(u_1)) ] = \Pr_{s \in F} [ s = r ] \cdot ( f(r(u_2)) - f(r(u_1)) ) > 0,$
so
$ \expect_{s \in D} [ f(s(u_1)) ] > \expect_{s \in F} [ f(s(u_1)) ] $ and therefore, by~\eqref{eq:blah2},
$ D[u_1] > F[u_1] .$
Moreover $\forall v \in \users \setminus \{u_1,u_2\}$ it holds that $F[v] = D[v]$. Therefore $D$ is a distribution improving the expected satisfaction of $u_1$ w.r.t. $F$ and such
that no $v \in \users$ exists such that $F[v] \leq F[u_1]$ and $D[v] < F[v]$, thus contradicting the assumption that $F$ is maxmin-fair. This proves~\eqref{eq:oneoranother}.

To prove~\eqref{eq:blah},
consider first the case $F[u_1] \ge F[u_2]$. Since $R(u_1) \ge R(u_2)$ implies $g(u_1) \ge g(u_2)$, in this case substituting $D = F$ in ~\eqref{eq:blah2} we trivially
obtain~\eqref{eq:blah}. If instead $F[u_1] < F[u_2]$ then, by~\eqref{eq:oneoranother}, we conclude that $f(r(u_1)) \ge f(r(u_2))$ for all $r \in F$, which
implies~\eqref{eq:blah}, as we wished to show.
\end{proof}

\subsection{Proof of Theorem~\ref{thm:lorenz}}
In this subsection we consider the case where we only have upper bounds in the group-fairness constraints.

First we need a result characterizing the minimum expected satisfaction of a maxmin-fair distribution. It has been inspired by the proof of~\cite{dgs2020}[Theorem
15]. While~\cite{dgs2020} only considers matroid problems (which do not cover constrained ranking), our key insight is that this type of argument can be generalized \emph{whenever there is a weight
    optimization oracle depending only on the weight order} (as opposed to the numerical values of the weights). This is true of the greedy algorithm from
    Corollary~\ref{lem:greedy} (Algorithm~\ref{alg:weighted_upper}).
\begin{lemma}\label{lem:characterization}
Let $\lambda: 2^{\calU} \to \reals$. There is a distribution of valid rankings such that $D[u] \ge \lambda_u$ if and only if
\begin{equation}\label{eq:condition}
 \max_{S \in \sols} \sum_{u \in X} A(S, u) \ge \sum_{u \in X} \lambda_u \qquad \text{ for all $X \subseteq \calU$} .
\end{equation}
\end{lemma}
\begin{proofof}{Lemma~\ref{lem:characterization}}
Given a set $E$, let $\Delta(E) = \{ x:E \to \reals^{\ge 0} \mid \sum_{e \in E} x_e = 1 \}$ denote the set of distributions over $E$.
Consider the following two-player zero-sum game: Player 1 (the maximizer) chooses a distribution of solutions $p \in \Delta(\sols)$, Player 2 (the minimizer) chooses a distribution of users $w \in
\Delta(\users)$, and the payoff for Player 1 when she plays $S \in \sols$ and Player 2 plays $u \in \users$ is 
$A(S, u) - \lambda_u$. The value of this game is
$$ v = \max_{p \in \Delta(\sols)} \min_{w \in \users} [ \sum_{S\in \sols} p_S (A(S, u) - \lambda_u) ];  $$
the required distribution exists when $v \ge 0$.
By Von Neumann's minimax theorem we have
\begin{equation}\label{eq:minmax}
v = \min_{w \in \Delta(\users)} \max_{S \in \sols} [ \sum_{u \in \users} w_u (A(S, u) - \lambda_u) ].
\end{equation}
Thus, $v \ge 0$ exactly when for all $w \in \Delta(\calU)$,  it holds that
\begin{equation}\label{eq:minmax2}
 \max_{S \in \sols} \sum_{u \in \users} w_u A(S, u) \ge \sum_{u\in \users} w_u \lambda_u.
\end{equation}
The result will follow if we can show that the minimization problem~\eqref{eq:minmax} has an optimal solution of the form
\begin{equation}\label{eq:unweighted}
w_u =
        \begin{cases}
          \frac{1}{|X|}, & \text{if } u\in X \\
          0, & \text{otherwise}
        \end{cases}
\end{equation}
for some non-empty $X \subseteq \calU$, because for $w_u$ of the form~\eqref{eq:unweighted}, \eqref{eq:minmax2} simplifies to~\eqref{eq:condition} on multiplication by $|X|$.
We have seen in Corollary~\ref{lem:greedy} that for each $w$,   $\max_{S \in \sols} \sum_{u\in \users} w_u (A(S, u) - \lambda_u)$ can be optimized by an oracle (Algorithm~\ref{alg:weighted_upper}) that only depends on the order
determined by $w$ (observe that subtracting $\lambda_u$ from $A(S, u)$ amounts to adding $\lambda$ to the function $g$ in the definition of $A(S, u)$). In other words, for any bijection
$\pi: [n] \to \users$ and any weight $w \ge 0$ \emph{compatible with $\pi$}
(i.e., satisfying $w_{\pi(1)} \ge w_{\pi(2)} \ge \ldots w_{\pi(n)}$), we have
$$ \max_{S \in \sols} \sum_{u\in \users} w_u (A(S, u) - \lambda_u) = \sum_{u \in \users} w_u ( A(G(\pi), u) - \lambda_u ), $$
where $G(\pi)$ is the solution returned by the greedy weighted optimization oracle.

Fix an order $\pi: [n] \to \users$ and let $B_u = A( G(\pi), u) - \lambda_u$. Consider the minimization problem
\begin{align}\label{eq:subproblem}
 \min \left\{  \sum_{u\in \users} w_u B_u \mid w \in \Delta(\users), \text{$w$ compatible with $\pi$} \right\}.
 \end{align}
Let $d_{n} = w_{\pi(n)}$ and $d_i = w_{\pi(i)} - w_{\pi(i+1)}$.
The compatibility conditions for $w$ may be rewritten as $d_i \ge 0$ for all $i$, and the distributional constraint $\sum_i w_i = 1$
becomes $\sum_i i \cdot d_i = 1$. If we write $z_i = \sum_{j \le i} B_{\pi(j)}$, then~\eqref{eq:subproblem} becomes
\begin{equation}\label{eq:ztt}
 \min \left\{ \sum_{i \in [n]} d_i\cdot z_i  \mid d_i \ge 0, \sum_{i\in[n]} i d_i = 1 \right\}  = \min \left\{ \frac{z_t}{t} \mid t \in [n] \right\};
\end{equation}
the last equality is easily seen to hold because $z_i \le \lambda\cdot i$ for all $i$ implies $\sum_i d_i z_i \le \lambda \sum_i i \cdot d_i = \lambda.$
Therefore for each $\pi$, an optimal solution to~\eqref{eq:subproblem} is of the form~\eqref{eq:unweighted}, where
$X = \{ {\pi(1)}, \ldots, {\pi(t)} \}$;
hence the same
is also true of an optimal solution to~\eqref{eq:minmax}.
\end{proofof}

\begin{corollary}\label{col:characterization}
The minimum expected satisfaction in any maxmin-fair distribution of valid rankings
is
$\ \min_{\emptyset \neq X\subseteq \calU} \frac{\max_{S \in \sols} \sum_{u \in \users} A(S, u)}{|X|} . $
\end{corollary}

We also need the following technical lemma concerning the behaviour of the expression  in Corollary~\eqref{col:characterization}.
\begin{lemma}\label{lem:submodular}
The following function $H: 2^{\calU} \to \reals$ is submodular:
\begin{equation}\label{eq:H}
H(E) = \max_{S \in \sols} \sum_{u \in E}  A(S, u).
\end{equation}
\end{lemma}
\begin{proof}
Let $J(E) = \max_{S \in \sols} \sum_{u \in E} f(S(u))$. Then
    $H(E) = J(E) - \sum_{u \in E} g(u)$, i.e.,
    $H$ is the difference between $J$ and a modular function. So it suffices to
show that $J$ is submodular; let us fix $X\subseteq Y$ and $Z\subseteq \calU\setminus Y$.
Recall from corollary~\ref{lem:greedy}  that $\sum_w w_u f(S(u))$ is maximized by a greedy algorithm. By setting $w_u = 1$ for $u \in E$ and $w_u = 0$ elsewhere, it can be
used to compute $J(E)$ for any $E$; let us denote by $r_E$ the ranking returned. Whenever we have two equal weights $w_u = w_v$, we can break ties in
Algorithm~\ref{alg:weighted_upper} in favor of $X$, followed by $Y\setminus X$,
     $Z$, and
        $\users \setminus (Y\cup Z)$. Then the
greedy algorithm to maximize $f(Y \cup Z)$
attempts to place the elements of $X$ in top positions whenever possible, then elements of $Y$, and then
elements of $Z$. This ensures that in
$r_X$ and $r_{X \cup Z}$
    the position of the elements of $X$ is the same, allowing us to simplify the
    marginal gains:
    \begin{align*} J(X \cup Z) - J(X) &= \sum_{u \in X \cup Z} f( r_{X \cup Z}(u) ) - \sum_{u \in X} f( r_{X}(u) )\\ &= \sum_{u \in Z} f( r_{X \cup Z}(u) ).  \end{align*}
Similarly,
\begin{align*}
 J(Y \cup Z) - J(Y) &= \sum_{u \in Y \cup Z} f( r_{Y \cup Z}(u) ) - \sum_{u \in Y} f( r_{Y}(u) )\\ &= \sum_{u \in Z} f( r_{Y \cup Z}(u) ).
\end{align*}
Moreover, for any $x \in Z$, $r_{X \cup Z}(x) \le r_{Y \cup Z}(x)$ by the greedy rule in Corollary~\ref{lem:greedy} and our tie-breaking rule.
 Therefore $f( r_{X \cup Z}(x) ) \ge f( r_{Y \cup Z}(x) )$, which implies
 $ J(X \cup Z) - J(X) \ge J(Y \cup Z) - J(Y)$, as desired.
\end{proof}

The following is an analog of the ``fair decompositions'' of~\cite{dgs2020}:
\begin{lemma}\label{lem:blocks}
Define a sequence of sets $B_1, B_2, \ldots, B_k$ iteratively by:
$ B_i \text{ is a maximal non-empty set } X \subseteq \users\setminus S_{i-1}  { minimizing } $ $$\frac{ H(X \cup S_{i-1}) - H(S_{i-1}) }{|X|},$$

\begin{equation}\label{eq:rule}
\text{where $H$ is given by~\eqref{eq:H} and  $S_i = \bigcup_{j \le i} B_i$.}
\end{equation}
We stop when $S_i = L$ (i.e., $k$ is the first such $i$); this will eventually occur as the sequence $(S_i)$ is strictly increasing.
Then for every $i\in[k]$, the following hold:
\begin{enumerate}[(a)]
    \item The expected satisfaction of any $u \in B_i$ in any maxmin-fair distribution $F$ is $F[u] = \frac{ H (B_i) }{ |B_i| }$.
    \item For all $u \in S_{i}, v \notin S_i$, we have $F[u] < F[v]$.
    \item For any $D \in \Delta(\sols)$ and
    $m \le |\users|$,
    it holds that $\sum_{j=1}^m D_{(j)} \le \sum_{j=1}^m F_{(j)}$.
\end{enumerate}
\end{lemma}
\begin{proof}
Let $\lambda_j = \frac{ H(B_j) } {|B_j|}$ for all $j \in [k]$.
Notice that, using~\eqref{eq:rule},
$ H(S_i) = H(S_{i-1} \cup B_i) = \lambda_i |B_i| + H(S_{i-1}) $
holds for all $i$,  hence
\begin{equation}\label{eq:sum_lambda}
 H(S_{i})  = \sum_{j=1}^{i} \lambda_j |B_j|.
\end{equation}
Notice also that  the definition of $H$ trivially implies that for any distribution $D$ we must have
\begin{equation}\label{eq:trivial}
 H(X) \ge \sum_{u \in X} D[u].
\end{equation}

We can give an alternative definition of $S_i$ as
\begin{equation}\label{eq:another}
 S_i \text{ is a maximal set } X \nsubseteq S_{i-1} \text{ minimizing }
 \frac{ H(X) - H(X \cap S_{i-1}) }{|X \setminus S_{i-1}|}.
\end{equation}
 Indeed, for any fixed difference $Y=X \setminus S_{i-1}$, the submodularity of $H$ (Lemma~\ref{lem:submodular})
    implies that the minimum of the numerator in~\eqref{eq:another} is attained for a set $X$ which properly contains $S_{i-1}$.
In particular we have that for any $X\subseteq \users$ and $j \in [k]$,
   $ H(X) - H(X \cap S_{j-1}) \ge \lambda_j |X \setminus S_{j-1}|, $
and by replacing $X$ with $X \cap S_j$ above we also get
   $ H(X \cap S_j) - H(X \cap S_{j-1}) \ge \lambda_j |X \cap B_j|, $
   implying
   \begin{equation}\label{eq:decomposition}
H(X) \ge \lambda_i |X \setminus S_{i-1}| +  \sum_{j < i}  \lambda_j |X \cap B_j| \qquad \forall i \in [k].
   \end{equation}

We show that properties (a) and (b) hold for all $i \le t \le k$, reasoning by induction on $t$.
There is nothing to show when $t = 0$ or $S_{i-1} = \calU$ (in the latter case, $k < i$), so assume that $t \ge 1$ and the claims hold for all $i < t$; we show they also hold for
$i = t$.

From property (a) in the induction hypothesis, we know that in the maxmin-fair distribution, $F[u] = \lambda_j \forall u \in S_{j}, j \le i$. We can use Lemma~\ref{lem:characterization} to determine the minimum expected
satisfaction of $F$ outside $S_{i-1}$; we conclude, by
~\eqref{eq:decomposition}, that
$\min_{u \notin S_{i-1}} F[u] \ge \lambda_i$.
As~\eqref{eq:sum_lambda} shows, equality in~\eqref{eq:decomposition} is attained when $X = S_i$, thus by~\eqref{eq:trivial}
we must in fact have
$\min_{u \notin S_{i-1}} F[u] = \lambda_i$ and $F[u] = \lambda_i$ for all $u \in B_i$, proving (a).

To prove (b), 
we need to show the strict inequality  $\lambda_{i-1} < \lambda_i$.
By Lemma~\eqref{lem:submodular}, the function $ J(X) = H(X \cup S_{i-1}) - H(S_{i-1}) $ is submodular.
A consequence of this is that,
    if $X, Y$ are non-empty sets
    minimizing $J(X)/|X|$, then $X \cup Y$ also
    minimizes $J(X)/|X|$.
    Indeed,
    suppose $\frac{ J(Y) }{ |Y| } = \frac{ J(X) }{ |X| } \triangleq \lambda$.
    By the submodularity of $J$, 
       $$ J(X \cup  Y) + J(X \cap Y) \le J(X) + J(Y) = \lambda (|X| + |Y|).$$
       Notice that $J(X \cup  Y) \ge \lambda |X \cup Y|$ and $ J(X \cap Y) \ge \lambda |X \cap Y| $ by definition. If any of these two inequalities were strict we would have
       the contradiction
    $$ J(X \cup  Y) + J(X \cap Y)  > \lambda (|X \cup Y| + |X \cap Y|) = \lambda (|X| + |Y|).$$
    Hence these inequalities are not strict, and $J(X \cup Y) = \lambda |X \cup  Y|$.

Now, due to the maximality of $B_i$ as defined by~\eqref{eq:rule}, the set $B_i$ 
    is
    the union of all non-empty sets $X$ minimizing $J(X)/|X|$.
    This means that, when $t > 1$, the strict inequality $\lambda_t > \lambda_{t-1}$ holds (otherwise $B_{t-1}$ would not be maximal), which by (a) implies (b).

Finally we show (c). We argue by contradiction. Pick a counterexample with minimum $m$; then $m \ge 1$. Let $i$ be such that $|S_{i-1}| < m \le |S_{i}|$. 
Then we have
$\sum_{j=1}^{m-1} D_{(j)} \le \sum_{j=1}^{m-1} F_{(j)}$
and
$\sum_{j=1}^{m} D_{(j)} > \sum_{j=1}^{m} F_{(j)},$
thus $D_{(m)} > F_{(m)} = \lambda_i$ by properties (a) and (b). Now let $X$ be the individuals with the $m$ smallest satisfactions in $D$. It follows that
\begin{align*}
 H(S_i) &\ge \sum_{u \in S_i} D[u]
\ge \sum_{u \in X} D[u] + (|S_i|-m) D_{(m)}   \\
     &= \sum_{j=1}^{m} D_{(j)} + (|S_i|-m) D_{(m)}
      > \sum_{j=1}^{m} F_{(j)} + (|S_i|-m)\lambda_i \\ &= \sum_{j \le i} \lambda_j |B_j| = H(S_i).
\end{align*}
This contradiction completes the proof.
\end{proof}

\begin{proofof}{Theorem~\ref{thm:lorenz}}
Property (c) of Lemma~\ref{lem:blocks} 
         asserts that the maxmin-fair distribution $F$ is
generalized Lorenz-dominant.
\end{proofof}

\subsection{Proof of Theorem \ref{lem:greedy2}}\label{app:deterministic}

\begin{proof}
Sort $\calU = \{u_1, \ldots, u_n\}$ by increasing order of $g$
so that
\begin{equation}\label{eq:good_order3}
g(u_1) \ge g(u_2) \ge \ldots \ge g(u_n),
\end{equation}
and let us identify $\calU$ with the set $[n]$ for ease of notation, so that $u_i = i$.
Recall that the positions $[n]$ are sorted by decreasing order of $f$ so that
\begin{equation}\label{eq:good_order4}
f(1) \ge f(2) \ge \ldots \ge f(n).
\end{equation}
Let $M > 0$ be a large enough number and define $W_{iu} = -e^{-M ( f(i) - g(u) ) }$. Observe that, because of the orderings defined by~\eqref{eq:good_order3} and~\eqref{eq:good_order4}, the matrix $W$ satisfies the Monge property:
if $i < j$ and $u < v$, then $W_{iu} + W_{jv} \ge W_{iv} + W_{ju}$. Indeed,
   $$
    W_{iu} + W_{jv} - ( W_{iv} + W_{ju} )
    = ( e^{-M f(j)} - e^{-M f(i)} ) ( e^{M g(u)} - e^{M g(v)} )
     \ge 0
     $$
     because $g(u) \ge g(v)$ and $f(i) \ge f(j)$, so both factors are non-negative. 
Thus we may apply the algorithm from~\cite{CelisSV18} to maximize $\sum_u W_{r(u), u}$ over valid rankings $r$. The resulting algorithm is
Algorithm~\ref{alg:baseline}. For any fixed $M$, maximizing $\sum_u W_{r(u), u}$ is the same as maximizing
 $ (1/M)\cdot{\softmin \{ M \cdot A(S, u) \mid {u \in \users} \} }. $
 But since the solution $S^* = S^*(M)$ returned by this algorithm does \emph{not} depend on $M > 0$ and
 $ \lim_{M \to \infty} \frac{\softmin\left( M\cdot z \right)}{M} = \min\left( z \right), $
 it follows that $S^*$ maximizes
 $ \min \{ A(S, u) \mid {u \in \users} \}, $
 as we wished to show.
\end{proof}

\subsection{Solving maxmin-fairness approximately}\label{sec:apx_fairness}
Instead of solving the LPs in the proof of Theorem~\ref{thm:general_algo} exactly, we can use iterative methods designed to approximately solve zero-sum games and packing/covering programs, as sketched next.

Recall that the exact algorithm~\ref{alg:exact} works by solving the linear program~\eqref{lp:fair3} and updating $K$ and $\alpha_v$.
Let us apply a positive affine transformation to normalize all $A(S, v)$ to the range $[0,1]$ and select an additive approximation parameter $\epsilon > 0$, so we want to ensure that in the final solution, the expected satisfaction for every $v \in \calU$ is at most an additive $\epsilon$ below that
which would have been computed by solving LP~\eqref{lp:fair3} exactly at the point where $\alpha_u$ was assigned.

Rather than maximizing $\lambda$ directly in~\eqref{lp:fair3}, we can guess an approximation $\tilde{\lambda}$ to the optimum, and verify if the guess is correct by eliminating the $\lambda$ variable from this LP and replacing
it with our guess $\tilde{\lambda}$, and then
checking if the program is feasible. Denote by $M$ the resulting LP. $M$ is a fractional covering program, equivalent to a zero-sum game, hence the techniques from~\cite{young1995randomized} apply.~If~$M$~is feasible, the algorithm
from~\cite{young1995randomized} returns a non-negative solution with
$\sum_s p_S = 1$ and violating the remaining constraints by at most an additive term $\delta$, using $O(\log n / \delta^2)$ calls to the separation oracle. This solution is sparse, having $O(\log n/ \delta^2)$ non-zero coefficients.
    By performing binary search on $\tilde{\lambda}$, we can
solve~\eqref{lp:fair3} up to a $\delta/2$ term in the satisfaction probabilities by approximately solving $O(\log(1/\delta))$ packing problems. Then we augment $K$ by adding those users where the satisfaction probabilities in the approximate solution
increased by at most $\delta / 2$; then we know than in the exact solution, it was impossible to increase then by more than $\delta$. This process, however, may decrease the satisfaction probabilities of individuals already in $K$ by up to $\delta
/ 2$. If we have solved $m$ different LPs by the time we reach $K = \calU$, the total cumulative error in the satisfaction probabilities in the final
solution and the optimum values of the program where they were computed is at most $\delta m$. Since $m \le O(n \log(1/\delta))$, we can take $\delta = O(\epsilon / (n \log (n/\epsilon))$ to guarantee that $\delta m \le \epsilon$.

Several improvements over this basic scheme can be made.
First, the above bound for $m$ is often too pessimistic, and it is more efficient to do a ``doubling trick'': start with $m = 2$ and keep doubling $m$ and restarting again with $K = \emptyset$ if during the execution of the algorithm sketched above we end up solving more than $m$ programs.
Second, we can use the variable-step increase technique from~\cite{packing_covering}.
Third, in the case of ranking problems with upper bounds, the separation oracle only depends on the order of the weights and not their specific values, so there is no need to call it again if this order does not change; we can simply increase the
probability of that solution.
Finally, for a given order of weights, ~\eqref{eq:ztt} allows us to obtain an optimal dual solution that respects that given weight order, which can be used to detect convergence of the iterative algorithm from~\cite{packing_covering} earlier.

\end{document}